\documentclass[preprint,12pt,authoryear]{article}

\usepackage{lastpage}
\usepackage[T1]{fontenc}
\usepackage[noend]{algpseudocode}
\usepackage{tikz,pgfplots,pgfplotstable,enumitem,amsmath,amssymb}
\usepackage{amsthm}
\usepackage[margin=1in]{geometry}
\usepackage{xspace,xcolor,algorithm,subfig}
\usepackage{natbib}
\usepackage{authblk}
\pgfplotsset{compat=1.14}

\newcommand{\rlsone}{\textproc{RLS$_1$}\xspace}
\newcommand{\rlstwo}{\textproc{RLS$_2$}\xspace}
\newcommand{\rlsk}{\textproc{RLS$_k$}\xspace}
\newcommand{\rlsm}{\textproc{RLS$_m$}\xspace}

\newcommand{\ARGk}{\textsc{ARG$_k$}\xspace}
\newcommand{\LeadingOnes}{\textsc{LeadingOnes}\xspace}
\newcommand{\OneMax}{\textsc{OneMax}\xspace}
\newcommand{\tmax}{\tau_{\max}}
\newcommand{\tend}{\tau_\mathrm{end}}
\newcommand{\tinit}{\tau_\mathrm{0}}
\newcommand{\popt}{p_\mathrm{opt}}
\newcommand{\pnop}{p_\mathrm{\neg opt}}
\newcommand{\TauBoundEvent}{L_6}
\newcommand{\N}{\ensuremath{\mathbb{N}}}
\newcommand{\Var}{\operatorname{Var}}

\colorlet{dgreen}{green!40!black}
\newtheorem{theorem}{Theorem}

\newtheorem{lemma}[theorem]{Lemma}
\newtheorem{corollary}[theorem]{Corollary}
\DeclareMathOperator{\Imp}{Imp}

\newcommand{\mquad}{{\mkern-18mu}}
\newcommand{\cnew}{\color{blue}}

\begin{document}
{\sloppy

\title{Selection Hyper-heuristics Can Automatically Adjust the Learning Period to Optimally Solve Pseudo-Boolean Problems\footnote{An extended abstract of this manuscript has appeared at the 2018 Genetic and Evolutionary Computation Conference (GECCO 2018) \citep{DLOW2018}}}

\author[1]{Benjamin Doerr}
\author[2]{Pietro S. Oliveto}
\author[3]{John Alasdair Warwicker}
\affil[1]{Laboratoire d'Informatique (LIX), CNRS, {\'E}cole Polytechnique, Institut Polytechnique de Paris, Palaiseau, 91120, France}
\affil[2]{Department of Computer Science and Engineering, Southern University of Science and Technology, Shenzhen, 518055, China}
\affil[3]{School of Computing \& Communications, Lancaster University Leipzig, Leipzig, 04109, Germany}

\date{}
\renewcommand\Affilfont{\itshape\small}

\maketitle

\begin{abstract}
The Random Gradient hyper-heuristic was recently shown to be able to learn the optimal neighbourhood size when optimizing the \LeadingOnes benchmark via the Randomised Local Search (RLS) meta-heuristic. 
However, for this to happen, a learning period of a certain length $\tau$ had to be used, differently from classic hyper-heuristics, which change their behaviour based on the success of only the previous iteration.
In this paper, we show how to automatically set this new parameter value, relieving the user from the non-trivial task of controlling this novel algorithm parameter.
We prove that the resulting hyper-heuristic selects the optimal neighbourhood size in a $1-o(1)$ fraction of the iterations and, consequently, optimises the \LeadingOnes benchmark in the best possible time (apart from lower-order terms) achievable with these neighborhood sizes. 
\end{abstract}


\section{Introduction}
Randomised search heuristics (RSHs)  have been successfully applied to numerous real-world optimisation problems. While their main strength is their problem independence (they can be applied to virtually any optimisation problem), it is well understood that each RSH will only be efficient on some classes of problems and not on others~\citep{WolpertMacready1997}. The tedious and time-consuming trial-and-error phases, used to determine which heuristics and parameter settings lead to good solutions for the problem at hand, are one of the main difficulties in the application of RSHs. 

The fields of automated algorithm design and hyper-heuristics (HHs) aim to automatically evolve the algorithm and related parameters for the problem, rather than making such choices manually. Successful applications of HHs have been reported for a variety of practical problems, including scheduling \citep{CowlingEtAl2000,CowlingEtAl2002}, timetabling \citep{OzcanEtAl2012} and vehicle routing \citep{AstaOzcan2014} (see \citep{BurkeEtAl2010,BurkeEtAl2013,PillayQuBook,DrakeEtAl2019} for extended surveys). However, a rigorous and foundational understanding of the behaviour and performance of HHs is largely lacking; the operations research community has acknowledged that such results are necessary for an informed advancement of the field \citep{BurkeEtAl2013}.

Recently some preliminary theoretical studies of {\it selection} HHs have appeared~\citep{LehreOzcan2013,AlanaziLehre2014,QianEtAl2016}.
Selection HHs consist of high-level mechanisms which select from a set of low-level heuristics which one should be applied in the next step of the optimisation process. A comparative analysis for the \LeadingOnes benchmark problem revealed that the simple Permutation, Greedy and Random Gradient selection HHs from the literature~\citep{CowlingEtAl2000,CowlingEtAl2002} have the same performance as that of simply choosing a random low-level heuristic at each step, i.e., they all require $(1+o(1))\frac{\ln(3)}{2}n^2\approx(1+o(1))0.549n^2$ fitness evaluations in expectation to optimise \LeadingOnes when choosing between Randomised Local Search (RLS) with neighbourhood size 1 (\rlsone) and neighbourhood size 2 (\rlstwo)~\citep{LissovoiEtAl2020ECJ}.

The reason for the disappointing performance was that the simple HHs only consider whether or not a heuristic was successful in the previous step to decide whether to apply it again in the following generation~\citep{LissovoiEtAl2020ECJ}. For most discrete optimisation problems and most RSHs, the probability that an operator is successful in a single iteration in a typical run is very low. As a result, even if the current best operator is chosen, it is very unlikely that the simple HHs will continue using it. Hence the HH will not exploit better performing operators.

\cite{LissovoiEtAl2020ECJ} proposed to modify the random gradient HH to measure the performance of operators over a fixed period of time $\tau$, rather than in a single iteration. They rigorously proved that the resulting HH, Generalised Random Gradient (GRG), outperforms its constituent low-level heuristics (i.e., $\{\rlsone,\dots,\rlsk\}$ where $k=\Theta(1)$) for \LeadingOnes provided that the learning period $\tau$ is large enough. In a subsequent analysis, they proved that for learning periods $\tau=\omega(n)$ and $\tau\leq (1/k - \varepsilon) n \ln n$, $\varepsilon>0$ a constant, GRG using $\{\rlsone,\dots,\rlsk\}$ where $k=\Theta(1)$ matches the best possible expected runtime achievable by any unbiased (1+1) black box algorithm using the same set of RLS low-level heuristics, up to lower order terms~\citep{LissovoiEtAl2020ECJ}. 

While the above result shows that hyper-heuristics with a longer learning period can optimally choose the currently best low-level heuristics, it also raises the question of how the typical user of this heuristic would find suitable values of the learning period~$\tau$. In particular, the proofs in the previous works suggest that suitable values for $\tau$ depend heavily on the optimization problem regarded (i.e., $\tau$ has to be large enough to make an informed decision, but small enough to prevent the algorithm from working for too long with sub-optimal low-level heuristics). 

To address this problem, in this paper, we design a way to automatically adapt the learning period $\tau$ throughout the run, and we prove that this gives the same asymptotically optimal result as the one obtained with the manually optimized values in~\cite{LissovoiEtAl2020ECJ}. We note that this not only removes the burden of having to identify an appropriate value for the learning period, but also has the potential to adjust the value of $\tau$ in situations where it changes considerably during the optimisation process. In particular, this extends to settings where no fixed value of $\tau$ is optimal (we do not explore this aspect further in this work, but are convinced that such situations exist).

More specifically, we equip the HH with an update scheme inspired by the 1/5 rule, where the value of $\tau$ is decreased by a multiplicative factor if the currently selected
low-level heuristic proves successful within a period of $\tau$ steps, and increased by a multiplicative factor if the low-level heuristic proves unsuccessful. The factors are chosen such that $\sigma$ decreases of $\tau$ counteract exactly one increase, reflecting, in essence, a $1-1/\sigma$ rule. To simplify the algorithmic settings, the HH (which we call Adaptive Random Gradient (ARG)) will consider a low-level heuristic {\it successful} if it obtains $\sigma$ improvements in a learning period $\tau$, rather than just one improvement as in previous work~\citep{LissovoiEtAl2020ECJ}.

Our main result is the proof that the presented ARG HH with low-level heuristic set $H=\{\rlsone,\dots,\rlsk\}$ has optimal performance for \LeadingOnes, i.e., no (1+1) black box algorithm using the same heuristics may be faster in expectation apart from lower order terms. In contrast to the parameter $\tau$ in the previous work, the hyper-parameter $\sigma$ here is not very critical - while our proof requires $\sigma\in \Omega(\log^4 n) \cap o(\sqrt{n / \log n})$, experiments show that ARG optimises the function faster than the GRG HH with optimal static values of $\tau$ for all tested values of $\sigma$ including those outside our theoretical range, i.e., it is faster also for practical problem sizes. 

Compared to its conference version which originally analysed ARG \citep{DLOW2018}, this manuscript has been considerably extended. The theoretical and experimental results, which previously held for heuristic sets of size 2, now hold for low-level heuristic sets of arbitrary size $|H|=k=\Theta(1)$, and we show that ARG with $k$ operators outperforms any algorithm using a strict subset of its operators. Furthermore, ARG now considers the low-level heuristic set consisting of RLS mutation operators, which do not allow the same bit to be flipped multiple times in a single mutation (in the conference version, the low-level heuristics selected each bit for mutation with replacement).

The rest of the paper is structured as follows. In the next section, we provide a brief overview of previous theoretical results concerning parameter adaptation and formally introduce the selection HHs considered in this paper. We present our runtime analysis of the ARG HH for \LeadingOnes in Section~\ref{sec:runtime-analysis}. Section~\ref{Sec:Experiments} presents the experimental results for practical problem sizes. We conclude the paper with a discussion of how ARG would also be efficient with even larger sets of low-level heuristics.

\section{Preliminaries}\label{sec:preliminaries}
In the next subsection, we provide an overview of the theoretical results concerning automated parameter control in evolutionary computation. In subsection \ref{subsec:hh}, we introduce the HHs that have been previously analysed theoretically in the literature, and we introduce the Adaptive Random Gradient (ARG) HH considered in this paper in subsection~\ref{subsec:argg}. 

\subsection{Parameter Adaptation}
Since adapting the step size is essential in evolutionary continuous optimisation, parameter adaptation is well-established within that field \citep{BeyerSchwefel2002}. Within discrete settings, it has been gaining attention in recent years. When the $(1+(\lambda,\lambda))$~GA was proposed, it was clear that the best value for the offspring population size $\lambda$ changed during the optimisation process~\citep{DoerrEtAl2015}. It was subsequently proven that adapting $\lambda$ using a 1/5 rule allows the algorithm to be asymptotically faster than any static parameter choice for \OneMax, and optimises the function in linear time~\citep{DoerrDoerr2017}. Further experimental results showed that the $(1+(\lambda,\lambda))$~GA is slower than RLS for problem dimensions up to $n=5{,}000$. However, by tuning the hyper-parameters inherent in the self-adjusting mechanism, the performance improves by $\approx15\%$ \citep{DangDoerr2019}.

Different parameter control schemes within discrete optimisation are classified based on the style of the parameter adaptation. A recent classification by \cite{DoerrDoerrBookChapter} suggested to differentiate between state-dependent, success-based, learning-inspired and self-adaptive parameter control. Additionally, hyper-heuristics, as online automated algorithm selection methodologies that learn through applying different low-level heuristics, can be classified as a form of high-level parameter control.

One of the first theoretical works on adaptive parameters in discrete optimisation was by \cite{LassigSudholt2011}, where self-adjusting mutation rates and population sizes were shown to improve the parallel runtime of an island model for several test functions. Concerning the $r$-valued \OneMax function, a self-adjustment of the (1+1)~EA mutation rate, inspired by the 1/5 rule, was proven to have the asymptotically best possible expected runtime achievable by unary (i.e., mutation-only) unbiased randomised search heuristics~\citep{DoerrEtAl2016A}. For the classic $(1+\lambda)$~EA, an adaptive choice of the mutation rate based on comparing the success of a smaller and a larger rate was shown to beat any static mutation rate~\cite{DoerrGWY19}. A more deterministic recent approach to escaping from local optima is stagnation detection, where the algorithm starts with low mutation rates and increases the rate when it is certain with high probability that the current rate cannot find improving solutions \citep{RajabiWitt2021,RajabiWitt2022,RajabiWitt2023,DoerrR23}. This approach is shown to escape local optima, e.g., of Jump functions, with the same efficiency which the (instance-dependent) optimal mutation rate would yield. A self-adjusting choice of the offspring population size was proven to let the $(1,\lambda)$~EA easily optimize Cliff functions~\citep{FajardoS24}.

Self-adaptation, where the parameter to be adjusted is encoded in the genome and the adjustment happens automatically in the usual variation and selection cycle, has also recently been analyzed. \cite{DangLehre2016} presented an example function for which, in some areas of the search space, a high mutation rate is required, while in other areas it is detrimental. They proved that a
generational stochastic selection EA with self-adaptation is efficient, while using static parameters leads to an exponential expected runtime. \citet{DoerrWY21} showed that a self-adaptive choice of the mutation rate gives the same, asymptotically optimal, performance as the complex self-adjusting mechanism designed in \cite{DoerrGWY19}. Self-adaptation was also shown to help in dynamic, noisy, and co-evolutionary optimization~\citep{LehreQ23gecco,LehreQ23foga,FajardoHTOL24}.

For a moderately recent, extensive overview of parameter control theory in discrete optimisation see~\citep{DoerrDoerrBookChapter}.

\subsection{Hyper-heuristics}\label{subsec:hh}
Although hyper-heuristics have found various successful applications, they are not yet fully understood. In particular, it is not clear which hyper-heuristic methodology and which move acceptance operator should be applied to a given optimisation problem, and which set of low-level heuristics should be used. Hence, a rigorous theoretical foundational understanding is required in order to provide insights into which problem classes hyper-heuristics perform efficiently and inefficiently \citep{BurkeEtAl2013}.

Let $S$ be a finite search space, $H$ a set of low-level heuristics and $f:S\rightarrow\mathbb{R}$ a cost function. Algorithm \ref{alg:SimpleHH} shows
the pseudocode for a simple selection hyper-heuristic (HH) as used in previous theoretical analyses \citep{AlanaziLehre2014,AlanaziLehre2016}.

\begin{algorithm}[t]
\caption{Simple Selection Hyper-Heuristic~\citep{AlanaziLehre2014,AlanaziLehre2016}}
\begin{algorithmic}[1]
\State Choose $x \in S$ uniformly at random
\While{stopping conditions not satisfied}
\State Choose $h \in H$ according to the learning mechanism
\State $x^\prime \gets h(x)$
\If{$f(x^\prime) > f(x)$}
\State $x \gets x^\prime$
\EndIf
\EndWhile
\end{algorithmic}
\label{alg:SimpleHH}
\end{algorithm}

In such works, the following simple learning mechanisms, commonly used in the literature to solve combinatorial optimisation problems, were considered \citep{CowlingEtAl2000, CowlingEtAl2002}: 
{\it Simple Random}, which selects a low-level heuristic $h\in H$ independently with probability $p_h$ in each iteration (usually $p_h=1/|H|$, i.e., uniformly at random);
{\it Permutation}, which generates a random ordering of low-level heuristics in $H$ and returns them in that sequence when called by the mechanism;
{\it Greedy}, which applies all available low-level heuristics in $H$ in parallel and returns the best found solution; and
{\it Random Gradient}, which randomly selects a low-level heuristic $h\in H$ and keeps using it iteratively as long as it obtains improvements.

\cite{LissovoiEtAl2020ECJ} analysed the four mechanisms on the \LeadingOnes (LO) benchmark function. LO is defined on the set $\{0,1\}^n$ of bit-strings of length $n$:
$$\textsc{LO}(x):=\sum_{i=1}^n\prod_{j=1}^i x_j.$$
They proved that the four mechanisms 
all have the same performance for LO up to lower-order terms. Essentially, all HHs choose low-level heuristics at random in each iteration. The authors hypothesised that the disappointing performance was due to the HHs not running the chosen low-level heuristic long enough to establish implicitly whether the selected choice was promising or not. To this end, they modified the simple Random Gradient HH such that the randomly chosen heuristic was run for a certain number of iterations $\tau$, which they called
the {\it learning period}. If the chosen heuristic found a fitness improvement within the learning period, then the period of $\tau$ iterations restarts. The pseudocode for the HH, {\it Generalised Random Gradient} (GRG), is presented in Algorithm~\ref{alg:SimpleHHWithTau}.

\begin{algorithm}[t]
\caption{Generalised Random Gradient Hyper-Heuristic~\citep{LissovoiEtAl2020ECJ}}
\begin{algorithmic}[1]
\State Choose $x \in S$ uniformly at random
\While{stopping conditions not satisfied}
\State Choose $h \in H$ uniformly at random
\State $c_t \gets 0$
\While{$c_t < \tau$}
\State $c_t \gets c_t+1$; $x^\prime \gets h(x)$
\If{$f(x^\prime) > f(x)$}
\State $c_t \gets 0$; $x \gets x^\prime$
\EndIf
\EndWhile
\EndWhile
\end{algorithmic}
\label{alg:SimpleHHWithTau}
\end{algorithm}

\cite{LissovoiEtAl2020ECJ} proved that if the learning period is set such that $\tau=\omega(n)$ and $\tau\leq(\frac{1}{k}-\varepsilon)n\ln n$ for $\varepsilon=\Theta(1)$, then GRG with $H=\{\rlsone,\dots,\rlsk\}$ (this set is often referred to as the \textit{initial segment} portfolio of RLS operators \citep{BiedenkappEtAl2022}) optimises LO in expected time matching the best possible expected runtime for any unbiased (1+1) black box algorithm using the same heuristics, up to lower order terms; that is, the best possible runtime achievable using any combination of the low-level heuristics. We refer to such performance as \textit{optimal} throughout the paper. Recently, it has been proved that GRG is also very effective at escaping local optima in multi-modal optimisation~\citep{Ma2025}.


A stagnation detection algorithm with a radius memory was recently introduced that, similarly to GRG, also saves the last successful rate in order to re-use it \citep{RajabiWitt2021}. The algorithm, though, is more sophisticated than GRG because the mutation rates are initially identified using stagnation detection rather than just chosen at random.

Regarding more sophisticated learning mechanisms, \cite{DoerrEtAl2016B} proposed one to estimate the efficiency of different parameter values from their medium-term past performance. This mechanism obtains the best possible runtime achievable by unary (i.e., mutation-only) unbiased randomised search heuristics for \OneMax by mutating the optimal number of bits in each step with high probability. Although not explicitly specified in the paper, this is the first work that rigorously analyses the performance of a hyper-heuristic that uses a sophisticated reinforcement learning mechanism to select low-level heuristics.
Recently, the stagnation detection algorithm of \cite{RajabiWitt2022,RajabiWitt2023} was updated to include an archive of promising mutation rates to be used in subsequent steps based on their past performance \citep{KrejcaWitt2024}. This algorithm also very closely resembles a hyper-heuristic by attempting to learn the most effective mutation rates for the problem at hand, although this is not explicitly mentioned in the paper. The benefit of the addition of an archive was shown on a hurdle problem with consecutive gaps of alternating sizes to be overcome. Thus, learning more than one useful mutation rate leads to a considerable speed-up compared to using a fixed mutation rate.

Apart from seeking to identify the optimal mutation strength, initial analyses on switching between elitist and non-elitist selection mechanisms to escape from local optima have also been recently undertaken \citep{LissovoiEtAl2023,DoerrEtAl2023,DoerrLutzeyer2024,DoerrIJCAI2025}. For an extensive overview of the theoretical analyses of hyper-heuristics, see \citep{OlivetoBookChapter}.

\subsection{The Adaptive Random Gradient Hyper-heuristic}\label{subsec:argg}
In this paper, we further generalise the Random Gradient HH presented by \cite{LissovoiEtAl2020ECJ} by enabling it to automatically adjust the learning period $\tau$ throughout the run. We modify GRG by introducing a simple self-adjusting mechanism inspired by the 1/5 rule from continuous optimisation~\citep{BeyerSchwefel2002}. The main idea is that the learning period $\tau$ should be large enough for the current best low-level heuristic to succeed, but also small enough such that sub-optimal heuristics fail.  

We define a heuristic to be {\it successful} if it achieves at least $\sigma$ fitness improvements during the learning period $\tau$. If fewer than $\sigma$ improvements occur (i.e., the heuristic fails), then $\tau$ is increased by a multiplicative factor $ F^{1/\sigma}$ and a new operator is chosen at random.
Otherwise, once the heuristic is successful, $\tau$ is decreased by a smaller factor $F^{1/\sigma^2}$ and a new learning period is started with the successful heuristic. This self-adjusting rule strives to adapt the learning period such that a failure occurs approximately every $\sigma$ successes, and to maintain a success probability of $1-1/\sigma$, i.e., a $1-o(1)$ rule. In contrast to the one success after several failures of traditional 1/5 rule algorithms~\citep{KernEtAl2004,DoerrDoerr2015}, the innovation behind the $1-o(1)$ rule is that it seeks many successes before a failure. The parameter $F$ should be a constant greater than 1 (previously $F=1.5$ has been successfully used~\citep{DoerrEtAl2015}), and the initialised value of the learning period (i.e., the parameter $\tinit$) should be at least 1. We recommend setting $\tinit=1$, since the exponential update rule can quickly grow the value of $\tau$, while setting the value too large might be disruptive to the initial performance of the HH.

The Adaptive Random Gradient HH (ARG) is formally described in Algorithm~\ref{alg:arg}. 

\begin{algorithm}[t]
\caption{Adaptive Random Gradient Hyper-Heuristic}\label{alg:ARG}
\begin{algorithmic}[1]
\State $\tau\gets\tinit$
\State Choose $x \in S$ uniformly at random
\While{optimum not found}
\State Choose $h\in H$ uniformly at random
\State  $c_t \gets 0$; $c_s \gets 0$
\While{$c_t<\tau$}\label{marker}
\State $c_t\gets c_t+1$; $x' \gets h(x)$
\If{$f(x')> f(x)$}
\State $c_s\gets c_s+1$; $x \gets x'$
\EndIf
\If{$c_s\geq\sigma$}
\State $c_s\gets 0$; $c_t\gets 0$
\State $\tau\gets \tau \cdot F^{-1/\sigma^2}$
\EndIf
\EndWhile
\State $\tau\gets\tau\cdot F^{1/\sigma}$
\EndWhile
\end{algorithmic}
\label{alg:arg}
\end{algorithm}

\section{Adaptive Random Gradient Optimises \LeadingOnes in Optimal Expected Time}\label{sec:runtime-analysis}
In this section, we will prove that the Adaptive Random Gradient HH, using the low-level heuristic set $H=\{\rlsone,\dots,\rlsk\}$ (where $k=\Theta(1)$), has optimal expected runtime for the \LeadingOnes (\textsc{LO}) benchmark function. 
We use $\rlsm$ to denote the randomised local search heuristic which samples a new solution with Hamming distance $m$ to the current solution.

\subsection{Theoretical Performance Limits for LeadingOnes}
The following lemmata provide statements regarding the improvement probabilities of unbiased (1+1) black box algorithms using RLS mutation operators for \textsc{LO}.

\begin{lemma}\label{moplemma}\citep{DoerrWagner2018,Doerr2019TCS}
The probability of improvement from an operator which flips $m$ distinct bits in a bit-string (\rlsm) on the \textsc{LO} benchmark function, at the state $\textsc{LO}(x)=i$, is
\begin{align*}
\Pr[\Imp_m\mid\textsc{LO}(x)=i]&=m\cdot\frac{1}{n}\cdot\prod_{j=1}^{m-1}\frac{n-i-j}{n-j}.
\end{align*}
\end{lemma}

\begin{corollary}\label{moplemma2}
Consider a randomised local search mutation operator flipping $a+1=\Theta(1)$ bits respectively (\textsc{RLS}$_{a+1}$), with $a\geq 1$. Then \textsc{RLS}$_{a+1}$ has a higher probability of a fitness improvement in one iteration than all \rlsm operators with $m\leq a$ (excluding $\textsc{LO}(x)\geq n-(a+1)$) when 
$$\textsc{LO}(x)\leq\frac{n-a}{a+1}.$$
\end{corollary}
Corollary~\ref{moplemma2} follows as an exact consequence of Lemma~\ref{moplemma}. 

For the \textsc{LO} function, an improvement produced by any \textsc{RLS} operator has the same expected fitness increase regardless of how many bits are flipped, giving no advantage to operators flipping more bits. An algorithm which always chooses the operator with the highest improvement probability will have the lowest expected optimisation time \citep{BoettcherEtAl2010}, and this will give a lower bound for any algorithm using the same operators. It follows from Corollary~\ref{moplemma2} that for an unbiased (1+1) black box algorithm using $\{\rlsone,\dots,\rlsk\}$ operators:
\begin{itemize}
    \item \rlsk will have the highest improvement probability for $0\leq \textsc{LO}(x)\leq \frac{n-(k-1)}{k}$,
    \item \rlsm ($1\leq m<k$) will have the highest improvement probability for $\frac{n-m}{m+1}\leq \textsc{LO}(x) \leq \frac{n-(m-1)}{m}$,
\end{itemize}
Hence, the expected runtime of an algorithm that applies these operators in this order provides a theoretical lower bound on the expected runtime for any unbiased (1+1) black box algorithm using the same operators. The following theorem states the best possible runtime for an algorithm using $\{\rlsone,\dots,\rlsk\}$ operators (where $k=\Theta(1)$) for \textsc{LO}. A similar result was proven for local search heuristics which can flip the same bit multiple times in the same mutation, and was subsequently extended to apply also for RLS heuristics (see Theorem~10, Lemma~16 and the subsequent text in \citep{LissovoiEtAl2020ECJ}).

\begin{theorem}\label{koptbestcase}\citep[Theorem~10, Lemma~16]{LissovoiEtAl2020ECJ}
The best-possible expected runtime of any unbiased (1+1) black box algorithm using $\{\rlsone,\dots,\rlsk\}$ operators (where $k=\Theta(1)$) for \LeadingOnes  is
\begin{align*}
E[T_{k,\mathrm{opt}}]=\frac{1}{2}\cdot&\left(\sum_{i=0}^{n/k-1}\frac{1}{\frac{k}{n}\cdot\prod_{j=1}^{m-1}\frac{n-i-j}{n-j}}+\sum_{x=1}^{k-1}\sum_{i=n/(x+1)}^{n/x-1}\frac{1}{\frac{x}{n}\cdot\prod_{j=1}^{x-1}\frac{n-i-j}{n-j}}\right).
\end{align*}
\end{theorem}

A closed-form result for $E[T_{k,\mathrm{opt}}]$ is difficult to present. However, we can consider specific values for different numbers of RLS operators. 
In particular, taking limits as $n\rightarrow\infty$, we have $E[T_{1,\mathrm{opt}}]=\frac{1}{2}n^2$, $E[T_{2,\mathrm{opt}}]=\frac{1+\ln(2)}{4}n^2\approx0.423n^2$,  $E[T_{3,\mathrm{opt}}]=\left(\frac{1}{3}+\frac{\ln(2)}{2}-\frac{\ln(3)}{4}\right)n^2\approx 0.405n^2$, $E[T_{5,\mathrm{opt}}]=\left(\frac{3721}{11520}+\frac{\ln(2)}{2}-\frac{\ln(3)}{4}\right)n^2\approx 0.394n^2$.  

\cite{DoerrWagner2018} suggest that $E[T_{\infty,\mathrm{opt}}]\approx0.388n^2$, a leading constant that was matched up to 3 decimal places by the expected runtime of GRG for LO with access to 18 low-level RLS heuristics \citep{LissovoiEtAl2020ECJ}.

\subsection{Analysis of Adaptive Random Gradient for \LeadingOnes}
In this subsection, we consider the runtime of ARG for \textsc{LO} using the low-level heuristic set $H=\{\rlsone,\dots,\rlsk\}$ (where $k=\Theta(1)$). We will show that with the parameter set $\tinit=1$, $\sigma\in\Omega(\log^4 (n))\cap o\left(\sqrt{{n}/{\log (n)}}\right)$ and $F>1$ a constant, ARG achieves optimal expected runtime. We refer to ARG with this low-level heuristic set and parameter setting as \ARGk.

\begin{theorem}\label{thm:ARG-runtime}
The expected runtime of \ARGk for \LeadingOnes, with $H=\{\rlsone,\dots,\rlsk\}$ (where $k=\Theta(1)$), $\tinit=1$, $\sigma\in\Omega(\log^4 (n))\cap o\left(\sqrt{{n}/{\log (n)}}\right)$ and $F>1$ a constant, is
$$E[T_{\textsc{ARG}_k}] \leq E[T_{k,\mathrm{opt}}] \pm o(n^2).$$
\end{theorem}

We first present an outline of the proof of Theorem~\ref{thm:ARG-runtime}, with the formal proof at the end of this subsection.

We refer to the operator with the greatest probability of producing an improvement at the current search point (i.e., the \rlsm mutation operator for
$\frac{n-m}{m+1}\leq i\leq\frac{n-(m-1)}{m}$ for $m<k$ and the \rlsk mutation operator for $i\leq\frac{n-(k-1)}{k}$) as the \emph{optimal} operator and use $\popt(i)$ to denote its improvement probability. Conversely, we call any other operator \emph{non-optimal} and use $\pnop(i)$ to denote the maximum probability of improvement of any non-optimal operator at the point $\textsc{LO}(x)=i$.

We call a period of at most $\tau$ iterations in which the mutation operator produces $\sigma$ \textsc{LO} improvements by mutation a \textit{successful phase} and a period of $\tau$ iterations in which the mutation operator produces less than $\sigma$ \textsc{LO} improvements a \textit{failed phase}. We will bound
$$T_{\textsc{ARG}_k} = T_\mathrm{mid} + T_\mathrm{S} + T_\mathrm{NS} + T_\mathrm{FP}$$
by bounding each of the four contributing components:
\begin{itemize}[leftmargin=*]
\item $T_\mathrm{mid}$, the number of iterations spent in the border regions where $\lvert \mathrm{LO}(x) - \frac{n-(m-1)}{m} \rvert < \beta n$ for $1<m\leq k$. $\mathrm{LO}(x)$ is the \textsc{LO} value of the current solution and $\beta = o(1)$,
\item $T_\mathrm{S}$, the number of iterations spent in successful phases applying the optimal operator outside the border regions,
\item $T_\mathrm{NS}$, the number of iterations spent in successful phases applying the non-optimal operators outside the border regions,
\item $T_\mathrm{FP}$, the number of iterations spent in failed phases outside the border regions.
\end{itemize}
To prove the theorem, we will bound $E[T_\mathrm{S}] \leq E[T_{k,\mathrm{opt}}]$ (i.e., that the expected number of iterations spent using the optimal operator in successful phases is smaller than the expected runtime if the optimal operator was used throughout), and then show that the expected values of the other contributing terms are at most $o(n^2)$.

Firstly, we prove that $E[T_\mathrm{mid}] = o(n^2)$ for any $\beta = o(1)$ in Lemma~\ref{lem:middle-is-easy}.
\begin{lemma} \label{lem:middle-is-easy}
Let $T_\mathrm{mid}$ be the number of iterations \ARGk spends in the border regions (i.e., when $\lvert \textsc{LO}(x) - \frac{n-(m-1)}{m}\rvert < \beta n$ for $1<m\leq k$, with $\beta \in o(1)$). Then $E[T_\mathrm{mid}] = o(n^2)$.
\end{lemma}
\begin{proof}
The improvement probabilities of all the \rlsm operators ($m=1,\dots,k$) in the heuristic set are monotone decreasing with respect to $\textsc{LO}(x)$; hence, the expected waiting times for an improvement are monotone increasing.

As the final border region ends at $\textsc{LO}(x) = (n-1)/2 + \beta n < 3n/4-1$,
the expected time for $\rlsk$ to find an improvement at $\textsc{LO}(x) = 3n/4 -1$ provides an upper bound on the expected waiting time for any of the \rlsm operators ($m=1,\dots,k$) to find an improvement within any of the border regions. In particular, $\rlsk$ will take expected time
$$\frac{1}{\Pr[\Imp_k\mid \textsc{LO}(x)=3n/4-1]}=\frac{n}{k\left(1/4\right)^{k-1}}=\mathcal{O}(n)$$
to find an improvement when $\textsc{LO}(x)=3n/4-1$.

The border regions contain a total of $\sum_{m=2}^k 2\beta n\leq2k\beta n= o(n)$ fitness values. The expected time to progress through each of the $o(n)$ fitness values is given by $\mathcal{O}(n)$. As ARG$_k$ will not accept mutations decreasing the fitness of the current solution, it will only need to construct an improvement from each fitness value within a border region at most once. There are $2k\beta n = o(n)$ fitness values within border regions, and hence
$$E[T_\mathrm{mid}]\leq 2k\beta n\cdot \frac{n}{k\left(1/4\right)^{k-1}}= o(n^2).\qedhere$$
\end{proof}

We define 
$$\tmax(i) := \left(1+\frac{4}{\ln n}\right)\cdot\sigma\cdot\frac{1}{\popt(i)},$$
where $\popt(i)$ is the improvement probability of the optimal operator at the state $\textsc{LO}(x)=i$. To bound $E[T_\mathrm{NS}]$ and $E[T_\mathrm{FP}]$, we will prove the following.
\begin{enumerate}[leftmargin=*]

\item With high probability (at least $1-n^{-c^\prime}$ for any constant $c^\prime>0$), $\tau$ remains below $\tmax(i)$ throughout the optimisation process (Lemma~\ref{lem:stay-below-tmax}).
\item While $\lvert \textsc{LO}(x) - \frac{n-(m-1)}{m} \rvert > \beta n$, for $\beta=o(1)$, $1<m\leq k$ and $\tau < \tmax(i)$, any non-optimal operator fails a phase with probability $1 - e^{-\Omega(\beta\sigma)}$ (Lemma~\ref{lem:while-below-tmax}).
\end{enumerate}
\begin{lemma} \label{lem:stay-below-tmax}
Let $\TauBoundEvent$ denote the event that \ARGk finds the global optimum before $\tau\geq\tmax(i)$ occurs, where $\textsc{LO}(x)=i$. Then $\Pr[\TauBoundEvent]\geq1-\mathcal{O}(n^{-c^\prime})$.
\end{lemma}
\begin{proof}
We prove the claim by showing that at any time of the algorithm ($\textsc{LO}(x)=i$) where we have $\tau\leq\tau_1(i)=\left(1+\frac{1}{\ln(n)}\right)\cdot\sigma\cdot\frac{1}{\popt(i)}$ and then failures increase $\tau$ over $\tau_1(i)$, the probability that the algorithm increases $\tau$ to over $\tau_2(i)=\tau_1(i)\cdot F^{(\log_2(n))^2/\sigma}<\tmax(i)$
or more without first reaching or going below $\tau_1(i)$ is at most $n^{-c}$ for any constant $c>0$. We then prove that, with probability at least $1-n^{-c^\prime}$ for any constant $c^\prime>0$, this does not happen at any point throughout the optimisation process.

We first show that $\tau_2(i)<\tmax(i)$. By a Maclaurin series argument, we have that
\begin{align*}
    F^{(\log_2(n))^2/\sigma}&=1+\frac{\ln(F)(\log_2(n))^2}{\sigma}+o\left(\frac{\ln(F)(\log_2(n))^2}{\sigma}\right)\\ &=1+\mathcal{O}\left(\frac{1}{(\log(n))^{2}}\right),
\end{align*}
since $\sigma=\Omega(\log^4(n))$. Hence:
\begin{align*}
\tau_2(i)&=\tau_1(i)\cdot F^{{(\log_2(n))^2}/{\sigma}}\\
&=\left(1+\mathcal{O}\left(\frac{1}{(\log(n))^{2}}\right)\right)\cdot\left(1+\frac{1}{\ln(n)}\right)\cdot\frac{\sigma}{\popt(i)} \\
&\leq\left(1+\frac{1}{\ln(n)}+o\left(\frac{1}{\log(n)}\right)\right)\cdot\frac{\sigma}{\popt(i)}\\
&<\left(1+\frac{4}{\ln(n)}\right)\cdot\frac{\sigma}{\popt(i)}=\tmax(i).
\end{align*}
To increase $\tau$ to $\tau>\tau_2(i)$ from $\tau\leq\tau_1(i)$, at least $(\log_2(n))^2$ failures are necessary. Hence, at least $(\log_2(n))^2$ random operator choices are necessary. With probability at least $1-n^{-\log_2(n)}$, the optimal operator is chosen at least once before $\tau$ is increased above $\tau_2(i)$.

When an improvement is constructed by ARG$_k$ on \LeadingOnes, the fitness of the current solution increases by $1$ plus the number of \emph{free riders}:
the number of $1$-bits immediately following the first mutated bit in the offspring solution. As the bits past the first $0$-bit in the current solution
remain uniformly distributed, an improving solution contains at most $1$ free rider bit in expectation. Applying the classic multiplicative Chernoff bound
(see e.g., \citep[Corollary 1.10 (c)]{Doerr2011bookchapter}) to the number of free riders encountered over the $\sigma$ improvements constructed in a
successful phase yields that with overwhelming probability fewer than $2\sigma$ free riders are observed. Hence, the \LeadingOnes value of the current solution
is increased by no more than $3\sigma$ in a successful phase with overwhelming probability.

Let $Y$ be a binomially-distributed random variable with parameters $\tau_1(i)$ and $\popt(i+3\sigma)$. The number of improvements produced by the
optimal operator in a period of $\tau\geq\tau_1(i)$ iterations stochastically
dominates $Y$. Applying a classic multiplicative Chernoff bound with $E[Y] \geq \mu:=\tau_1(i)\cdot\popt(i+3\sigma)=\left(1+1/{\ln(n)}\right)\cdot\left(({n-3\sigma})/{n}\right)\cdot\sigma=\Theta(\sigma)$ and
$\delta:= \frac{n-3\sigma\ln(n)-3\sigma}{(1+\ln(n))(n-3\sigma)}=\Theta(1/\log(n))$ (such that $(1-\delta)\mu = \sigma$) yields
\begin{align*}
\Pr[Y<\sigma]\leq \Pr[Y\leq\sigma]&\leq\exp\left(-\left(\frac{n-3\sigma\ln(n)-3\sigma}{(1+\ln(n))(n-3\sigma)}\right)^2\cdot\left(1+\frac{1}{\ln(n)}\right)\cdot\frac{n-3\sigma}{n}\cdot\frac{\sigma}{2}\right)\\
&=\exp\left(-\Theta\left(\frac{\sigma}{\log^2(n)}\right)\right)
=n^{-\Omega(\log(n))},
\end{align*}
since $\sigma=\Omega(\log^4(n))$. 

ARG requires $\sigma\cdot (\log_2(n))^2$ consecutive successes by the optimal operator to return $\tau$ to below $\tau_1(i)$, cancelling out the increase in $\tau$ given by at most $(\log_2(n))^2$ failures from the non-optimal operators. Given a success probability of at least $1-n^{-\Omega(\log(n))}$, the probability of $\sigma\cdot (\log_2(n))^2$ consecutive successes given that each success contributes a fitness increase of at most $6\sigma$ (denoted as the event $Y^\prime$) is, by a union bound,
\begin{align*}
\Pr[Y^\prime]&\geq \left(1-(1-n^{-\Omega(\log(n))})\cdot n^{-\Omega(\log(n))}-n^{-\Omega(\log (n))}\right)^{\sigma (\log_2(n))^2}\\
&\geq 1-2\cdot n^{-\Omega(\log(n))}\cdot\sigma (\log_2(n))^2=1-n^{-\Omega(\log(n))}.
\end{align*}

Hence, the probability of $\tau$ exceeding $\tau_2(i)<\tmax(i)$ before returning below $\tau_1(i)$ (denoted as the event $Z$) is at most 
\begin{align*}
\Pr[Z]&\geq\left(1-n^{-\Omega(\log(n))}\right)\cdot\left(1-n^{-\log_2(n)}\right)\\
&\geq 1-2n^{-\Omega(\log(n))} .
\end{align*}

There will be at most $n/\sigma$ successful phases throughout the optimisation process and hence at most $n/\sigma$ times $\tau<\tau_1(i)$ occurs (which invokes the previous argument). Hence, the probability of $\tau$ not exceeding $\tmax(i)$ throughout the optimisation process (denoted as the event $\TauBoundEvent$) is at most
\begin{align*}
\Pr[\TauBoundEvent] & \geq\left(1-2n^{-\Omega(\log(n))}\right)^{n/\sigma}
\geq 1-2n^{-\Omega(\log(n))}\cdot n/\sigma \\
& = 1-n^{-\Omega(\log(n))},
\end{align*}
which is greater than $1-n^{-c^\prime}$ for any constant $c^\prime>0$ when
$n$ is sufficiently large.
\end{proof}
\begin{lemma}\label{lem:while-below-tmax}
Consider a run of \ARGk for \LeadingOnes. While $\tau < \tmax(i)$ and $\lvert \textsc{LO}(x) - \frac{n-(m-1)}{m} \rvert > \beta n$ for all $m=1,\dots,k$, where $\beta=1/\sqrt[4]{\ln(n)}$,
any non-optimal mutation operator produces at least $\sigma$ improvements within
$\tau$ iterations with probability at most $o(1)$.
\end{lemma}
\begin{proof}
We know from Corollary~\ref{moplemma2} that the \rlsm operator is the optimal operator when 
$$\frac{n-m}{m+1}\leq \textsc{LO}(x)\leq \frac{n-(m-1)}{m}.$$
Furthermore, as a direct consequence of Lemma~\ref{moplemma} and Corollary~\ref{moplemma2}, we have that the \textproc{RLS$_{m-1}$} and \textproc{RLS$_{m+1}$} operators are the next best operators within this region (since \textproc{RLS$_{m+1}$} is the second best operator at the start of the region and \textproc{RLS$_{m-1}$} is the second best operator at the end of the region). Thus, proving the claim of the lemma for these two operators within the region 
$$\frac{n-m}{m+1}+\beta n\leq \textsc{LO}(x)\leq \frac{n-(m-1)}{m}-\beta n$$
for $1\leq m\leq k$ will prove the lemma statement. Note that only one side of the argument is necessary for \rlsone and \rlsk.

We will show at the beginning of the region that \textproc{RLS$_{m+1}$} fails to produce at least $\sigma$ improvements within $\tau$ iterations with probability at least $o(1)$, noting that the results will hold for all RLS operators with larger neighbourhood sizes. Since the improvement probability of \textproc{RLS$_{m+1}$} shrinks faster than $\tmax$ grows throughout the region, the expected number of improvements produced by \textproc{RLS$_{m+1}$} in $\tmax$ iterations at the beginning of the region provides an upper bound on the expected number of improvements found throughout the region. We will repeat the argument by showing at the end of the region that \textproc{RLS$_{m-1}$} fails to produce at least $\sigma$ improvements within $\tau$ iterations with probability at least $o(1)$, noting that the results will hold for all RLS operators with smaller neighbourhood sizes. Similarly, the value of $\tmax$ grows faster than the improvement probability of \textproc{RLS$_{m-1}$} shrinks, so the expected number of improvements found by \textproc{RLS$_{m+1}$} in $\tmax$ iterations is bounded from above by the expected number of improvements at the end of the region.

We begin by proving the claim for \textproc{RLS$_{m+1}$} at the beginning of the region, showing that the operator is unsuccessful with high probability $1-o(1)$ when $\textsc{LO}(x)=\frac{n-m}{m+1}+\beta n$. We refer to \textproc{RLS$_{m+1}$} as the \textit{non-optimal operator}. 

Let $Y$ be a binomially-distributed random variable with parameters 
\begin{align*}
    &\tmax\left(\frac{n-m}{m+1}+\beta n\right)=\left(1+\frac{4}{\ln(n)}\right)\cdot\sigma\cdot\frac{n\left(n-\frac{n-m}{m+1}-\beta n-m\right)!(n-1)!}{m\left(n-\frac{n-m}{m+1}-\beta n-1\right)!(n-m)!}
\end{align*} and 
\begin{align*}
&\pnop\left(\frac{n-m}{m+1}+\beta n\right)=\frac{(m+1)\left(n-\frac{n-m}{m+1}-\beta n-1\right)!(n-m-1)!}{n\left(n-\frac{n-m}{m+1}-\beta n-m-1\right)!(n-1)!}.
\end{align*}
Since $\pnop(i)$ is a decreasing function, the number of improvements produced by the non-optimal operator in a period of $\tau<\tmax\left(\frac{n-m}{m+1}+\beta n\right)$ iterations is stochastically dominated by $Y$. Applying the classic multiplicative Chernoff bound with $\mu := E[Y]  =
\left(1+\frac{4}{\ln(n)}\right)\cdot\sigma\cdot\frac{(m+1)(n-\frac{n-m}{m+1}-\beta n-m)}{m(n-m)}$ and $\delta:=\frac{((4-4\beta)m-4\beta)n-n\beta(m+1)\ln (n)-4m^2}{(4+\ln (n))(((\beta-1)m+\beta)n+m^2)}=\Theta(\beta)$ (such that $(1+\delta)\mu = \sigma$) yields
\begin{align*}
	&\Pr[Y \geq \sigma] \leq \exp\Bigg(\\
 & -\left(\frac{((4-4\beta)m-4\beta)n-n\beta(m+1)\ln (n)-4m^2}{(4+\ln(n))(((\beta-1)m+\beta)n+m^2)}\right)^2\\
&\cdot\left(\left(1+\frac{4}{\ln(n)}\right)\cdot\sigma\cdot\frac{(m+1)(n-\frac{n-m}{m+1}-\beta n-m)}{m(n-m)}\right)\cdot\frac{1}{3}\Bigg)\\
	&=\exp(-\Theta(\beta^2\sigma))= o(1).
\end{align*}

We now repeat the argument for the other side of the region, by proving that \textproc{RLS$_{m-1}$} is unsuccessful with high probability $1-o(1)$ when $\textsc{LO}(x)=\frac{n-(m-1)}{m}-\beta n$. We now refer to \textproc{RLS$_{m-1}$} as the \textit{non-optimal operator}.

\begin{sloppypar}
Let $Y$ be a binomially-distributed random variable with parameters 
$\tmax\left(\frac{n-(m-1)}{m}-\beta n\right)=\left(1+\frac{4}{\ln(n)}\right)\cdot\sigma\cdot\frac{n\left(n-\frac{n-(m-1)}{m}+\beta n-m\right)!(n-1)!}{m\left(n-\frac{n-(m-1)}{m}+\beta n-1\right)!(n-m)!}$ and 
$\pnop\left(\frac{n-(m-1)}{m}-\beta n\right)=\frac{(m-1)\left(n-\frac{n-(m-1)}{m}+\beta n-1\right)!(n-m+1)!}{n\left(n-\frac{n-(m-1)}{m}+\beta n-m+1\right)!(n-1)!}$. Since $\pnop(i)$ is a decreasing function, the number of improvements
produced by the non-optimal operator in a period of $\tau<\tmax\left(\frac{n-(m-1)}{m}-\beta n\right)$ iterations is
stochastically dominated by $Y$. Applying the classic multiplicative Chernoff bound with $\mu := E[Y]  =
\left(1+\frac{4}{\ln(n)}\right)\cdot\sigma\cdot\frac{(m-1)(n-m+1)}{m(n-\frac{n-(m-1)}{m}+\beta n-m+1)}$ and $\delta:=\frac{4(m-1)(n-m+1)+\beta mn\ln (n)}{(4+\ln (n))(m-1)(n-m+1)}=\Theta(\beta)$ (such that $(1+\delta)\mu = \sigma$) yields
\begin{align*}
	&\Pr[Y \geq \sigma] \leq \exp\Bigg(-\left(\frac{4(m-1)(n-m+1)+\beta mn\ln (n)}{(4+\ln(n))(m-1)(n-m+1)}\right)^2\\
	&\cdot\left(1+\frac{4}{\ln(n)}\right)\cdot\sigma\cdot\frac{(m-1)(n-m+1)}{m(n-\frac{n-(m-1)}{m}+\beta n-m+1)}\cdot\frac{1}{3}\Bigg)\\
	&=\exp(-\Theta(\beta^2\sigma))= o(1).
\end{align*}
\end{sloppypar}

The previous arguments will also hold for RLS operators with neighbourhood size larger than \textproc{RLS$_{m+1}$}, and smaller than \textproc{RLS$_{m-1}$} respectively. Thus, any non-optimal operator will fail with high probability $1-o(1)$ to produce at least $\sigma$ improvements within $\tau$ steps within the region $\lvert \textsc{LO}(x) - \frac{n-(m-1)}{m} \rvert > \beta n$ for $m=1,\dots,k$, proving the Lemma statement.
\end{proof}

We can now bound $E[T_\mathrm{NS}]$ and $E[T_\mathrm{FP}]$, conditioned on the event that $\tau<\tmax(i)$ holds throughout the optimisation process.
\begin{lemma}\label{lem:unsuccessful-or-failed}
Consider a run of \ARGk for \LeadingOnes. Let $\TauBoundEvent$ denote the event that $\tau$ remains below $\tmax(i)$ throughout the optimisation process. While $\lvert\textsc{LO}(x)-\frac{n-(m-1)}{m}\rvert>\beta n$ for $1<m\leq k$, and assuming that $\TauBoundEvent$ holds,
\begin{enumerate}
    \item The expected number of iterations \ARGk spends in failed phases is $E[T_\mathrm{FP}\mid\TauBoundEvent]=o(n^2)$;
    \item The expected number of iterations \ARGk spends in successful phases applying any non-optimal operator is $E[T_\mathrm{NS}\mid\TauBoundEvent]=o(n^2)$.
\end{enumerate}
\end{lemma}

\begin{proof}
For $T_\mathrm{FP}$ and $T_\mathrm{NS}$, bounds on the number of successful phases with the non-optimal operators as well as the total number of failed phases are needed. Let $N_\mathrm{S}$ and $N_\mathrm{F}$ denote the number of successful and failed phases (respectively) that occur before the global optimum is constructed and $\tend$ be the value of $\tau$ when the global optimum is constructed. $N_\mathrm{F}$ can be bounded by observing that, regardless of the order of the phases,
the following balance relation is valid:
$\tend  = \tinit F^{N_\mathrm{F}/\sigma - N_\mathrm{S}/\sigma^2}$. 

Conditional on $\TauBoundEvent$, we have $\tend < \max_{0\leq i\leq n}\tmax(i)\leq \tmax(n) = (1+o(1))\sigma n$ and hence, given $\tinit=1$,
\begin{align*}
    \log_\mathrm{F}(\tmax(n)) & \geq N_\mathrm{F}/\sigma - N_\mathrm{S}/\sigma^2.
\end{align*}

As each successful phase provides at least $\sigma$ fitness improvements, we can bound $N_\mathrm{S} \leq n/\sigma$ and then bound $N_\mathrm{F}$:
\begin{align*}
    N_\mathrm{F} & \leq N_\mathrm{S}/\sigma + \sigma \log_\mathrm{F} (\tmax(n))
    = \mathcal{O}(n/\sigma^2 + \sigma \log (n)).
\end{align*}

This bounds the number of iterations spent in failed phases:
\begin{align*}
E[T_\mathrm{FP} \mid \TauBoundEvent]
	&\leq \tmax(n) \cdot N_\mathrm{F} 
\\  &\le (1+o(1)) \sigma n \cdot \mathcal{O}(n/\sigma^2 + \sigma \log (n)) = o(n^2).
\end{align*}

Except during the iterations spent in the border regions of the search space (which are counted by $T_\mathrm{mid}$ and bounded by Lemma~\ref{lem:middle-is-easy}), the probability that a non-optimal operator produces $\sigma$ improvements within a single phase is at most $o(1)$ by Lemma~\ref{lem:while-below-tmax}. Conditional on $\TauBoundEvent$, there are at most $N_\mathrm{S}+N_\mathrm{F} = \mathcal{O}(n/\sigma)$ phases in total and hence at most $o(1)\cdot \mathcal{O}(n/\sigma)= o(n/\sigma)$ successful phases with the non-optimal operators, which, combined with $\tau < \tmax(i) \leq \tmax(n)$ yields a bound on the number of iterations spent in successful phases using the non-optimal operators:
\[ E[T_\mathrm{NS} \mid \TauBoundEvent]
	\leq \tmax(n) \cdot o(n/\sigma) = o(n^2).\qedhere \]
\end{proof}

We now present a coarse upper bound on the expected runtime of \ARGk for \LeadingOnes that does not require $\tau<\tmax(i)$ to hold throughout. Before this, we present the following technical lemma, which allows us to bound the probability that a process on the non-negative integers, which in expectation decreases in value, remains above its initial value after a certain amount of time. A similar result has been proven by \cite{KotzingLW15}. As witnessed by the $\ln\left(\frac{1}{(a+1)p}\right)$ term, the bound presented in Lemma~\ref{loccprob} is stronger when the expected movement away from zero is significantly smaller than the expected movement towards zero. We need this stronger result in Lemma~\ref{lem:unbounded-tmax}.
\begin{lemma}\label{loccprob}
  Let $X_0, X_1, \dots$ be a random process on the non-negative integers. Assume that there are $a \in \N_{\ge 1}$ and $p \in (0,\frac 1{e(a+1)})$ such that for all $t$ and all $k \ge 1$, we have $\Pr[X_{t+1} = X_t + a \mid X_t = k] \le p$ and $\Pr[X_{t+1} = X_t - 1 \mid X_t = k] = 1 - \Pr[X_{t+1} = X_t + a \mid X_t = k]$. Assume further that $X_0 = 0$. Then for all $k \in \N$ and all $t \in \N$, we have
$\Pr[X_t \ge k] \le \exp\left(-\frac{(k-1)(1-(a+1)p)}{a(a+1)}\left(\ln\left(\frac{1}{(a+1)p}\right)-1\right)\right) K (a+1),$
where $K$ is an absolute constant.
\end{lemma}
\begin{proof}
  Let $k, t \in \N_{\ge 1}$. To have $X_t \ge k$, there must be a $t_0 \in [0..t]$ such that $X_{t_0} = 1$ and $X_{t'} \ge 1$ for all $t' \in [t_0,t]$. Let $Y_t, t \in [t_0..t]$ be a random process with 
  \begin{itemize}
  \item $Y_{t_0} = 1$, 
  \item $\Pr[Y_{t+1} = Y_t + a \mid Y_t = k] = \Pr[X_{t+1} = X_t + a \mid X_t = k] \le p$ and $\Pr[Y_{t+1} = Y_t - 1 \mid Y_t = k] =\Pr[X_{t+1} = X_t - 1 \mid X_t = k]$ for all $k \in \N_{\ge 1}$ and
  \item $\Pr[Y_{t+1} = Y_t + a \mid Y_t = k] = p$ and $\Pr[Y_{t+1} = Y_t - 1 \mid Y_t = k] = 1-p$ for all $k \le 0$.
  \end{itemize}
  Note that this process is obtained from the process $(X_t \mid X_{t_0} = 1)$ by modifying it only on the non-positive integers. Consequently, $\Pr[X_{t} \ge k \wedge \forall t' \in [t_0,t] : X_{t'} \ge 1 \mid X_{t_0} = 1] = \Pr[Y_{t} \ge k \wedge \forall t' \in [t_0,t] : Y_{t'} \ge 1] \le \Pr[Y_{t} \ge k]$. Note that $Y_t$ is stochastically dominated by a random variable $\tilde Y_t = 1 + \sum_{t' = t_0}^{t-1} Z_{t'}$ with independent $Z_{t'}$ such that $\Pr[Z_{t'} = a] = p$ and $\Pr[Z_{t'} = -1] = 1-p$. We have $E[\tilde Y_t] = 1 + (t-t_0)(pa - (1-p))$, $\Var[\tilde Y_t] = \sum_{t'=t_0}^{t-1} \Var[Z_{t'}] \le (t-t_0) a^2 p$ and trivially $Z_{t'} \le E[Z_{t'}] + (a+1)$ for all $t'$. Note that for $t_0 > t - (k - 1)/a$, we trivially have $\Pr[X_{t} \ge k \mid X_{t_0} = 1] = 0$. Hence let $(t-t_0) \ge (k-1)/a$. By the above, we have $\Pr[X_{t} \ge k \wedge \forall t' \in [t_0,t] : X_{t'} \ge 1 \mid X_{t_0} = 1] \le \Pr[\tilde Y_t \ge k] \le \Pr[\tilde Y_t \ge 1] = \Pr[\tilde Y_t \le E[\tilde Y_t] - \lambda]$ for $\lambda = -(t-t_0)(pa - (1-p))$. Putting $b:=a+1$, we use the variance-based Chernoff bound (see the classic paper by \cite{Hoeffding63} or Theorem~1.12 and the subsequent text in~\citep{Doerr2011bookchapter}) and compute, writing $p = c/(a+1)$,
\begin{align*}  
  &\Pr[\tilde Y_t \le E[\tilde Y_t] - \lambda]\\
  &\le \exp\left(\!-\frac{\lambda}{b}\left(\left(1+\frac{\Var[\tilde Y_t]}{b\lambda}\right) \ln\left(1+\frac{b\lambda}{\Var[\tilde Y_t]}\right)-1\right)\right)\\
 	&\le \exp\left(\!-(t-t_0)\frac{1-p(a+1)}{a+1}\right.
    	\left.\left( \ln\left(1\!+\!\frac{(a+1)(1-p(a+1))}{a^2 p}\right)\!-\! 1\right)\right)\\
 	&\le \exp\left(\!-(t-t_0)\frac{1-c}{a+1}\big(\ln\left(\tfrac 1c\right)-1\big)\right).
\end{align*}
Hence,
\begin{align*}
\Pr[X_t \ge k] &\le \mquad \sum_{t_0 = 0}^{t-(k-1)/a} \mquad \Pr[X_{t} \ge k \wedge \forall t' \in [t_0,t] : X_{t'} \ge 1 \mid X_{t_0} = 1]\\
&\le \sum_{t_0 = 0}^{t-(k-1)/a} \Pr[\tilde Y_t \le E[\tilde Y_t] - \lambda]\\
&\le \sum_{\delta = 0}^{\infty} \exp\left(-(\delta+(k-1)/a)\frac{1-c}{a+1}\left(\ln\left(\frac 1c\right)-1\right)\right)\\
&\le \exp\left(-\frac{(k-1)(1-c)}{a(a+1)}\left(\ln\left(\frac 1c\right)-1\right)\right)\cdot 
  \sum_{\delta = 0}^{\infty} \exp\left(-\frac{1-c}{a+1}\left(\ln\left(\frac 1c\right)-1\right)\right)^\delta\\
& = \exp\left(-\frac{(k-1)(1-c)}{a(a+1)}\left(\ln\left(\frac 1c\right)-1\right)\right) K (a+1),
\end{align*}
where $K$ can be chosen as an absolute constant (independent from $a$ and $c$, provided that $c < 1/e$).
\end{proof}

\begin{lemma} \label{lem:unbounded-tmax}
Consider a run of \ARGk started with an arbitrary initial search point $x$ and an arbitrary initial period length $\tinit \le n^{k+1}$, for \LeadingOnes. Let $T$ be the runtime, that is, the number of fitness evaluations performed up to the point when for the first time the optimal solution is evaluated. Then $E[T] = \mathcal{O}(n^{k+1})$.
\end{lemma}
\begin{proof}
While the fitness is less than $n-k$, each of the $k$ operators has a probability of at least $1/n^k$ of finding an improvement. Hence by a simple fitness level argument (see~\citep{Oliveto2025,Jansen2013}, the time $T_0$ to reach a fitness of at least $n-k$ satisfies $E[T_0] \le n^{k+1}$. 

Throughout this first part of the optimisation process, we have that a learning period starting with a $\tau$-value of $n^{k+1}$ or more is successful with probability $1-\exp(-\Theta(n/\sigma))=:1-p$. This is because the expected number of improvements is at least $\tau/n^k$, whereas for a success we need only $\sigma$ improvements. Hence the standard multiplicative Chernoff bound for geometrically distributed random variables (see, e.g., \cite{Doerr2011bookchapter}) shows this claim.

Let $i$ be minimal such that $\tau' := \tinit F^{i/\sigma^2} \ge n^{k+1}$. In other words, $\tau'$ is the smallest $\tau$ value not smaller than $n^{k+1}$ which we could encounter in this run of the algorithm. For any time $t$, let $X_t = \max\{0,\log_{F^{1/\sigma^2}}(\tau_t / \tau')\}$. In other words, if $\tau_t \ge \tau'$, then $X_t$ is such that $\tau_t = \tau' (F^{1/\sigma^2})^{X_t}$; otherwise $X_t = 0$. By definition, we have $X_0 = 0$. Also, for all $t \ge 0$ and all $j \ge 1$, we have $\Pr[X_{t+1} = X_t + \sigma \mid X_t = j] \le p$ and $\Pr[X_{t+1} = X_t - 1 \mid X_t = j] = 1 - \Pr[X_{t+1} = X_t + \sigma \mid X_t = j]$. By Lemma~\ref{loccprob}, we have 
\begin{align*}
\Pr[X_t \ge j] &\le \exp\left(-\frac{(j-1)(1-(\sigma+1)p)}{\sigma(\sigma+1)}\cdot
\left(\ln\left(\frac{1}{(\sigma+1)p}\right)-1\right)\right) \Theta(\sigma+1)\\
&=\exp(-\Theta(nj/\sigma^3))
\end{align*}
for all $t, j \ge 1$, where the implicit constants can be chosen independently of $n, j, \sigma$ provided that $n$ is sufficiently large. We use this to compute $E[\tau_{T_0}] \le \sum_{j=0}^\infty \Pr[X_{T_0} = j] \tau' F^{j/\sigma^2} \le \tau' +  \sum_{j=1}^\infty \tau' F^{j/\sigma^2} \exp(-\Theta(nj/\sigma^3))= \mathcal{O}(\tau') = \mathcal{O}(n^{k+1})$.

We shall use this estimate of $\tau_{T_0}$, the value of $\tau$ used when the fitness first reached or exceeded $n-k$. If at time $T_0$ we already have a fitness of~$n$, then $T_1=0$ and there is nothing to show. So let us assume that the fitness at time $T_0$ is $n-k$ and bound the expectation of $T_1$ from above by estimating the time it takes to optimise the last $k$ bits.

As $k=\Theta(1)$ and $\sigma = \Omega(\log^4(n))$, no learning phase started after a fitness of at least $n-k$ has been reached can be successful (as the
global optimum would be constructed before $\sigma > k$ improvements can be). Thus, it is either the case that the global optimum is found within $2
\tau_{T_0}$ iterations (in which case $T_{1} \leq 2 \tau_{T_0}$), or a random operator choice following a failed phase has occurred. In the latter case, let
$\tau^{(1)} \le \tau_{T_0} F^{1/\sigma}$ denote the value of $\tau$ used in the learning period following this first random operator choice.

Let $P_i$ denote the $i$'th learning period following the first random operator choice once a fitness of $n-k$ has first been reached, $\tau^{(i)}$ denote the
value of $\tau$ used in this period, and $R_i = \sum_{j=i}^\infty r_i$ denote the probability that the optimum is found during period $P_i$ or later. Let
$i_0 \ge 1$ be minimal such that $\tau^{(i_0)} \geq n$. Given that the global optimum has not yet been constructed, the probability that at least one
improvement is constructed during period $P_i$ where $i \ge i_0$ is at least $q_i \geq (1/k)(1 - (1-1/n)^{\tau^{(i)}}) \geq (1/k)(1 - 1/e)$.
For an upper bound on $R_i$, we let $q_i = 0$ while $i < i_0$, and suppose that each learning phase that produced at least one improvement produced
\emph{exactly one} improvement, and thus at most $k-1$ phases prior to $P_i$ are allowed to produce improvements. Consequently, $R_i \le \min\{1,(1-q_i)^{i
- i_0 - k+1}\}$, and 
\begin{align*}
E[T_1] \leq 2\tau_{T_0} + \sum_{j=1}^{\infty} R_j \tau^{(j)}
\end{align*}

For $i\leq i_0 + (k+1)$, $R_i=1$. Furthermore, for $i<i_0$, $\tau^{(i)}\leq n$. Hence we can bound these terms:
$$\sum_{j=1}^{i_0+(k-1)} R_j\tau^{(j)} \leq (i_0-1)\cdot n + (k-1)\cdot\tau^{(1)}\cdot F^{(k-1)/\sigma}.$$

It remains to bound $\sum_{j=i_0+k}^{\infty}R_j\tau^{(j)}$. We estimate
\begin{equation*}
\begin{array}{ll}
\displaystyle\sum_{j=i_0+k}^{\infty} R_j \tau^{(j)} & \le \sum_{j=1}^\infty \left(1-\left(\frac{1}{k}\right)\left(1-\frac{1}{e}\right)\right)^{j} F^{j/\sigma} \tau^{(i_0)} \\
& \le C \tau^{(i_0)}
\end{array}
\end{equation*}
for some constant $C$ (assuming that $n$ is sufficiently large). Note that $\tau^{(i_0)} \le \max\{nF^{1/\sigma}, \tau^{(1)}\} \leq n^{k+1} F^{1/\sigma}$. Hence,
\begin{align*}
E[T_1] &\leq 2\tau_{\tau_0} + (i_0-1)\cdot n + (k-1)\cdot F^{(k-1)/\sigma} \cdot \tau^{(1)} + C\cdot n^{k+1}\cdot F^{1/\sigma}\\
& = \mathcal{O}(n^{k+1}).
\end{align*}

Thus, \ARGk constructs the global optimum for \LeadingOnes within $E[T_0] + E[T_1] = \mathcal{O}(n^{k+1})$ iterations in expectation.
\end{proof}

We can now prove Theorem~\ref{thm:ARG-runtime}, which states that \ARGk has optimal runtime for \LeadingOnes.

\begin{proof}[Proof of Theorem~\ref{thm:ARG-runtime}]
Recall that we will bound 
$T_{\textsc{ARG}_k} = T_\mathrm{mid} + T_\mathrm{S} + T_\mathrm{NS} + T_\mathrm{FP}$
by bounding each of the four contributing components:
\begin{itemize}[leftmargin=*]
\item $T_\mathrm{mid}$, the number of iterations spent in the border regions, where $\lvert \mathrm{LO}(x) - \frac{n-(m-1)}{m} \rvert < \beta n$, $1<m\leq k$ and
$\mathrm{LO}(x)$ is the \LeadingOnes value of the current solution. 
\item $T_\mathrm{S}$, the number of iterations spent in successful phases applying the optimal operator outside the border regions,
\item $T_\mathrm{NS}$, the number of iterations spent in successful phases applying the non-optimal operators outside the border regions,
\item $T_\mathrm{FP}$, the number of iterations spent in failed phases outside the border regions.
\end{itemize}
Lemma~\ref{lem:middle-is-easy} gives $E[T_\mathrm{mid}]=o(n^2)$. For $T_\mathrm{S}$, we note that $E[T_\mathrm{S}] \leq E[T_{k,\mathrm{opt}}]$ (i.e., the result of Theorem~\ref{koptbestcase}), as the optimum would be found in expectation after $E[T_{k,\mathrm{opt}}]$ iterations of applying the optimal mutation operator, while ARG can additionally make progress toward the optimum applying the non-optimal operators, as well as in periods which fail to produce $\sigma$ improvements. Let $\TauBoundEvent$ denote the event that $\tau$ remains below $\tmax(i)$ throughout the optimisation process. Lemma~\ref{lem:unsuccessful-or-failed} states that $E[T_\mathrm{NS}\mid\TauBoundEvent]=o(n^2)$ and $E[T_\mathrm{FP}\mid\TauBoundEvent]=o(n^2)$ respectively.

Combining the four contributing factors, by the linearity of expectation, the expected runtime of ARG is:
\begin{align*}
	&E[T_{\textsc{ARG}_k} \mid \TauBoundEvent]\\
    & \leq E[T_\mathrm{S} \mid \TauBoundEvent] + E[T_\mathrm{FP} \mid \TauBoundEvent] + E[T_\mathrm{NS} \mid \TauBoundEvent] + E[T_\mathrm{mid} \mid \TauBoundEvent] \\
    & \leq E[T_{k,\mathrm{opt}}] + o(n^2),
\end{align*}
noting that the bounds on the expected values of $T_\mathrm{S}$ and $T_\mathrm{mid}$ derived previously also hold when conditioned on
$\TauBoundEvent$.

For an unconditional expectation, we use Lemma~\ref{lem:unbounded-tmax} to bound the expected runtime of ARG when $\TauBoundEvent$ fails to hold as $\mathcal{O}(n^{k+1})$. By Lemma~\ref{lem:stay-below-tmax}, the probability of $\tau$ exceeding $\tmax(i)$ at any point before the global optimum is found can be made $n^{-c'}$ small for any constant $c' > 0$ and hence, for e.g. $c' = k$:
\begin{align*}
	E[T_{\textsc{ARG}_k}] & = \Pr[\TauBoundEvent]\cdot E[T_{\textsc{ARG}_k} \mid \TauBoundEvent]
    + \Pr[\overline{\TauBoundEvent}]\cdot E[T_\mathrm{ARG} \mid \overline{\TauBoundEvent}]
  \\& = E[T_{k,\mathrm{opt}}] + o(n^2) + \mathcal{O}(n^{-k})\cdot \mathcal{O}(n^{k+1})
  \\& = E[T_{k,\mathrm{opt}}] + o(n^2),
\end{align*}
as $\sigma= o(\sqrt{n / \log (n)})$.
\end{proof}
Note that the proof above not only shows that \ARGk has (apart from lower-order terms) the optimal runtime achievable with the given set of operators, but also that with high probability $1-o(1)$, only an $o(1)$ fraction of the iterations use the non-optimal operators. Hence, our proposed self-adjusting mechanism manages to select the most suitable mutation operator in an extremely effective manner. 

\subsection{Adaptive Random Gradient Successfully Handles Larger Sets of Heuristics}
We now prove that the expected runtime of \ARGk for \LeadingOnes is smaller than the expected runtime of any unbiased (1+1) black box algorithm using $m<k$ 
RLS low-level heuristics. As a direct result, we see that the performance of \ARGk improves when it has access to larger sets of low-level heuristics.
\begin{sloppypar}
\begin{theorem}\label{besttheorem}
For any integer $k > 1$, the expected runtime of \ARGk using $H=\{\rlsone,\dots,\rlsk\}$ for \LeadingOnes is smaller in the leading constant than the best-possible expected runtime for \LeadingOnes for any unbiased (1+1) black box algorithm using any strict subset of $H$.
\end{theorem}
\end{sloppypar}
\begin{proof}
Theorem~\ref{thm:ARG-runtime} shows that \ARGk matches the best possible performance for an algorithm with the same $k$ mutation operators (given in Theorem~\ref{koptbestcase}), up to lower order terms. 

\begin{sloppypar}
Consider an algorithm with access to $k-1$ mutation operators, $H^\prime=\{\rlsone,\dots,\textsc{RLS}_{k}\}\backslash\rlsm$. Theorem~\ref{koptbestcase} and Theorem~\ref{thm:ARG-runtime} imply that the only difference between the best possible expected runtime of such an algorithm and the expected runtime of \ARGk, occurs when the fitness value satisfies
$$\frac{n-m}{m+1}\leq\textsc{LO}(x)\leq\frac{n-(m-1)}{m}.$$
If $m=k$, then the difference occurs when $0\leq\textsc{LO}(x)\leq \frac{n-(k-1)}{k}$.
\end{sloppypar}

In these regions (for the given \rlsm operator), the expected runtime of \ARGk matches the expected runtime (up to lower order terms) of an algorithm applying only \rlsm in this region. That is,
$$\frac{1}{2}\cdot \sum_{i=(n-m)/(m+1)}^{(n-(m-1))/m}1/(\Pr[\text{Imp}_m\mid \textsc{LO}(x)=i]).$$

A closed-form expression is difficult to present. From Corollary~\ref{moplemma2}, we know that \rlsm is the fastest operator in this region, and optimises the region faster than any other sub-optimal operator, or combination of sub-optimal operators. For $k=\Theta(1)$, the order of this speedup is $\Omega(n^2)$ in expectation (as a direct result of Lemma~\ref{moplemma} and Corollary~\ref{moplemma2}).

We can extend this argument for the best possible performance of an algorithm using any strict subset of $H$. The expected runtimes of the best-possible algorithm would match \ARGk in regions where they share the same operators, while \ARGk is faster by $\Omega(n^2)$ in expectation where the algorithm does not have access to the optimal operator.
\end{proof}

As an example of the application of Theorem~\ref{besttheorem}, consider \ARGk with $k=3$, and a best-case algorithm using $H^\prime=\{\rlsone,\text{RLS}_3\}$. The expected runtimes of the two algorithms will match, except when $(n-2)/3 \leq \textsc{LO}(x) \leq (n-1)/2-1$. In this region, ARG$_3$ will use the \rlstwo with high probability, and will have expected runtime $\frac{2\ln2-\ln3}{4}n^2+o(n^2)\approx 0.07192n^2+o(n^2)$. Using exclusively either \rlsone or RLS$_3$ in this region gives an expected runtime of $\frac{1}{12}n^2+o(n^2)\approx 0.08333n^2$, while the best combination of the two operators gives an expected runtime of $\frac{2\sqrt{3}-3}{6}n^2+o(n^2)\approx 0.07735n^2$. Hence, ARG$_3$ is faster overall by $\approx 0.00543n^2$ in expectation.

\cite{BiedenkappEtAl2022} have recently shown that other sets of low-level RLS heuristics of the same cardinality $k$ may lead to faster runtimes if they have access to larger neighbourhood sizes than $k$. In particular, if the portfolio had cardinality $k$ (i.e., $k$ different RLS heuristics) yet the largest neighbourhood size was strictly greater than $k$ (e.g., $H=\{\rlsone, \rlstwo, \dots, \textsc{RLS}_{k+2}, \textsc{RLS}_{k+5}\}$ with $|H|=k$), it could be faster than the initial segment portfolio with neighbourhood sizes $1,\dots,k$.
We note that if \ARGk has access to different sets of low-level RLS heuristics (i.e., portfolios) other than the \textit{initial segment} portfolio, it would still achieve the best possible runtime available for these portfolios.  

\section{Experimental Supplements}\label{Sec:Experiments}
In the previous section we have proven that for large enough problem sizes $n$, ARG has the optimal possible performance for \LeadingOnes, up to lower order terms. In this section, we perform a set of experiments to assess the performance of ARG with $H=\{\rlsone, \ldots, \rlsk\}$ for practical problem sizes. Our theoretical results rely on the adaptation parameter $\sigma$ to grow slowly with the problem size, i.e., asymptotically between $\Omega(\log^4 n)$ and $o(\sqrt{n/\log n})$.
We experiment with various super-constant values within and outside the theoretical bounds.
We introduce multiplicative constants $c^*$ such that $\sigma$ is sufficiently small compared to $n$, for small problem sizes. We decide to arbitrarily set $c^*$ such that for our smallest problem size ($n=10^3$), $\sigma=4$. We set the initial learning period trivially to $\tinit=1$. We then set $F=1.5$ as suggested by \cite{DoerrDoerr2015} and run all algorithms 100 times.

In Figure~\ref{fig:n-increase}, we plot the average runtimes of ARG as the problem size increases and compare their performance against GRG with fixed $\tau$ values (noting that both achieve the same optimal performance based on the theoretical analyses for sufficiently large problem sizes $n$). We pick the best performing versions of GRG identified by \cite{LissovoiEtAl2020ECJ}  (i.e., $\tau=0.4n\ln n$, $0.5n\ln n$, and the best, $0.6n\ln n$, for the problem size $n=100{,}000$). 
While the runtimes of all the ARG variants are comparable with the three GRG variants, many outperform them, particularly as the number of low-level heuristics $k$ increases. For $k\geq3$, all of the ARG variants outperform the GRG variants. These include all the ones with $\sigma$ values within the theoretical bounds (i.e., $\sigma=c^*\ln^4n$) and various $\sigma$ values outside the bounds (i.e., $\sigma=\sqrt{n}/\ln n$, $\sigma=c^*\sqrt{n}/\ln n$ and $\sigma=c^*\sqrt{n/\ln n}$).
Hence, for realistic problem sizes, better runtimes can be achieved by adapting $\tau$ during the run, especially as the size of the heuristic set increases, highlighting the advantage of automatic parameter adaptation compared to static parameters. Furthermore, while a value of $\sigma$ within the derived range is required to prove optimal runtime, fast performance is achieved experimentally by a wider range of $\sigma$ values.

\begin{figure}[t!]
\centering
\pgfplotstableread{Experiments2/fig1a.tsv}\dtA
\pgfplotstableread{Experiments2/fig1b.tsv}\dtB
\pgfplotstableread{Experiments2/fig1c.tsv}\dtC
\pgfplotstableread{Experiments2/fig1d.tsv}\dtD
\def\xmax{2089296131}
\pgfplotsset{max space between ticks=25,
	xmin=1000, xmax=\xmax, ymin=0.392, ymax=0.502,
	width=0.37\linewidth,height=2in,scale only axis,
	scaled ticks=false,tick label style={/pgf/number format/fixed},
	ymajorgrids=true,
	y tick label style={
		/pgf/number format/.cd,fixed,fixed zerofill,precision=2,
		/tikz/.cd
	},
	xtick placement tolerance=-0.1pt,
}
\begin{tabular}{cc}
\begin{tikzpicture}[trim axis right]
\begin{semilogxaxis}[
    xlabel=$n$, ylabel=Runtime $(/n^2)$,
    legend style={font=\scriptsize, at={(0.99, 0.96)}},
    legend to name={fig1legend},
	 legend columns=3,
]
\addplot[thick,dgreen]          table [x index=0,y index=1] {\dtA};
\addplot[thick,red]             table [x index=0,y index=2] {\dtA};
\addplot[thick,cyan]            table [x index=0,y index=3] {\dtA};
\addplot[thick,brown]           table [x index=0,y index=4] {\dtA};
\addplot[thick,blue]            table [x index=0,y index=5] {\dtA};

\addplot[red,dotted]            table [x index=0,y index=6] {\dtA};
\addplot[cyan,dotted]           table [x index=0,y index=7] {\dtA};
\addplot[blue,dotted]           table [x index=0,y index=8] {\dtA};

\draw [thick, draw=gray]   (axis cs: 1000,0.5) -- (axis cs: \xmax,0.5)
node[pos=1,below left,inner sep=1pt] {\rlsone};
\draw [thick, draw=gray]   (axis cs: 1000,0.4232867952) -- (axis cs:\xmax,0.4232867952) node[pos=0,above right,inner sep=1pt] {$2_{opt}$
};

\legend{$\sigma=\sqrt{n}/\ln n$, $\sigma=c^*\ln^4 n$, $\sigma=c^* \sqrt{n}/\ln n$, $\sigma=c^* \sqrt{n/\ln n}$, $\sigma=c^* \sqrt{n}$,
$\tau=0.4n\ln n$,$\tau=0.5n\ln n$, $\tau=0.6n\ln n$}
\end{semilogxaxis}
\end{tikzpicture}
 &
\begin{tikzpicture}[trim axis left,trim axis right]
\begin{semilogxaxis}[xlabel=$n$,yticklabels={,,}]
\addplot[thick,dgreen]          table [x index=0,y index=1] {\dtB};
\addplot[thick,red]             table [x index=0,y index=2] {\dtB};
\addplot[thick,cyan]            table [x index=0,y index=3] {\dtB};
\addplot[thick,brown]           table [x index=0,y index=4] {\dtB};
\addplot[thick,blue]            table [x index=0,y index=5] {\dtB};
                                                                
\addplot[red,dotted]            table [x index=0,y index=6] {\dtB};
\addplot[cyan,dotted]           table [x index=0,y index=7] {\dtB};
\addplot[blue,dotted]           table [x index=0,y index=8] {\dtB};

\draw [thick, draw=gray]   (axis cs: 1000,0.5) -- (axis cs: \xmax,0.5)
node[pos=1,below left,inner sep=1pt] {\rlsone};
\draw [thick, draw=gray]   (axis cs: 1000,0.40525) -- (axis cs:\xmax,0.40525) node[pos=0,above right,inner sep=1pt] {$3_{opt}$
};
\end{semilogxaxis}
\end{tikzpicture}\\
(a) $k=2$ & (b) $k=3$ \\[6pt]
\begin{tikzpicture}[trim axis right]
\begin{semilogxaxis}[xlabel=$n$,ylabel=Runtime $(/n^2)$]
\addplot[thick,dgreen]          table [x index=0,y index=1] {\dtC};
\addplot[thick,red]             table [x index=0,y index=2] {\dtC};
\addplot[thick,cyan]            table [x index=0,y index=3] {\dtC};
\addplot[thick,brown]           table [x index=0,y index=4] {\dtC};
\addplot[thick,blue]            table [x index=0,y index=5] {\dtC};
                                                                
\addplot[red,dotted]            table [x index=0,y index=6] {\dtC};
\addplot[cyan,dotted]           table [x index=0,y index=7] {\dtC};
\addplot[blue,dotted]           table [x index=0,y index=8] {\dtC};

\draw [thick, draw=gray]   (axis cs: 1000,0.5) -- (axis cs: \xmax,0.5)
node[pos=1,below left,inner sep=1pt] {\rlsone};
\draw [thick, draw=gray]   (axis cs: 1000,0.39830) -- (axis cs:\xmax,0.39830) node[pos=0,above right,inner sep=1pt] {$4_{opt}$
};

\end{semilogxaxis}
\end{tikzpicture}
& 
\begin{tikzpicture}[trim axis left,trim axis right]
\begin{semilogxaxis}[xlabel=$n$,yticklabels={,,}]
\addplot[thick,dgreen]          table [x index=0,y index=1] {\dtD};
\addplot[thick,red]             table [x index=0,y index=2] {\dtD};
\addplot[thick,cyan]            table [x index=0,y index=3] {\dtD};
\addplot[thick,brown]           table [x index=0,y index=4] {\dtD};
\addplot[thick,blue]            table [x index=0,y index=5] {\dtD};
                                                                
\addplot[red,dotted]            table [x index=0,y index=6] {\dtD};
\addplot[cyan,dotted]           table [x index=0,y index=7] {\dtD};
\addplot[blue,dotted]           table [x index=0,y index=8] {\dtD};

\draw [thick, draw=gray]   (axis cs: 1000,0.5) -- (axis cs: \xmax,0.5)
node[pos=1,below left,inner sep=1pt] {\rlsone};
\draw [thick, draw=gray]   (axis cs: 1000,0.39492) -- (axis cs:\xmax,0.39492) node[pos=0,above right,inner sep=1pt] {$5_{opt}$
};
\end{semilogxaxis}
\end{tikzpicture}\\
(c) $k=4$ & (d) $k=5$ \\[6pt]
\end{tabular}
\vspace{2mm} \ref{fig1legend}
\caption{Average number of fitness function evaluations required by the ARG and GRG hyper-heuristics using $H=\{\rlsone, \ldots, \rlsk\}$ (with parameters $\sigma$ and $\tau$ respectively) to find the \LeadingOnes optimum as the problem size $n$ increases. \rlsone and $k_{opt}$ show the best runtime achievable by using, respectively, one low-level heuristic (since \rlsone is the only low-level heuristic that can make progress at $\textsc{LO}(x)=n-1$ and thus has a finite expected runtime) or all $k$ heuristics (see Theorem~\ref{koptbestcase}).}
\label{fig:n-increase}
\end{figure}

Figure~\ref{fig:typical-run-max} shows the adaptation of $\tau$ for five individual runs for $n=10^7$ of ARG with $\sigma=\sqrt{n}/\ln n$ and $H=\{\rlsone,\rlstwo\}$, 
as well as an average over 100 runs. As expected, the learning period $\tau$
is quickly increased
into the range where the optimal heuristic produces more than $\sigma$ improvements per
$\tau$ iterations in expectation. The adapted learning period tracks the increasing waiting times for an 
improvement by the optimal heuristic in the first half of the search space, while remaining
stable in the second half where the waiting time does not change.
$\tmax$ is included to show how the parameter $\tau$ adapts as suggested within the proof of Theorem~\ref{thm:ARG-runtime}.
The average runtime of ARG over 100 runs in this setting was $\approx0.426 n^2$ (the best possible expected runtime achievable is $\approx0.423 n^2$).

\begin{figure}[t!]
\pgfplotstableread{Experiments2/fig2a.tsv}{\dtA}
\begin{tikzpicture}
\begin{axis}[
	xlabel=\% $\textsc{LO}(x)$, xmin=0, xmax=100, ymin=50, ymax=255,
	ylabel=$\tau/n$,
	scaled ticks=false,tick label style={/pgf/number format/fixed},
	mark size=1.25pt, ymajorgrids=true,
	legend pos=south east,
	y tick label style={
		/pgf/number format/.cd,fixed,fixed zerofill,precision=0,
		/tikz/.cd
	},
	width=1\linewidth,
	height=2.2in,
	legend pos = south east
]
\addplot[red]                table[x index=0,y index=2] {\dtA};
\addplot[orange]             table[x index=0,y index=3] {\dtA};
\addplot[green]              table[x index=0,y index=4] {\dtA};
\addplot[brown]              table[x index=0,y index=5] {\dtA};
\addplot[purple]             table[x index=0,y index=6] {\dtA};
\addplot[black,thick,dotted] table[x index=0,y index=1] {\dtA};

\addplot[thick,mark=none,dashed, blue] coordinates {(49,1.248168275*196) (100,1.248168275*196)};

	\draw[thick,domain=0:49,dashed,variable=\x,blue] plot ({\x},{1.248168275*196*50/(99-\x)});
      \draw[blue] (47, 240) node[left] {$\tau_{max}$};
	
\legend{Run 1,Run 2,Run 3,Run 4,Run 5,Average}
\end{axis}
\end{tikzpicture}

\caption{Adapted value of $\tau$ over time in five typical runs of ARG using $H=\{\rlsone,\rlstwo\}$ for \LeadingOnes with $\sigma=\sqrt{n}/\ln n$, $n=10^7$ and an average over 100 runs.}
\label{fig:typical-run-max}
\end{figure}

Figure~\ref{fig:average-opt-use} shows the percentage of iterations where the optimal operator is used by ARG with $\sigma=\sqrt{n}/\ln n$ for problem sizes $n=10^5$, $n=10^7$ and $n=10^9$, averaged over 100 independent runs. We divide the $n+1$ fitness values into 100 ranges and plot the percentage of iterations where the optimal operator is employed.
We see that ARG exhibits the behaviour predicted by  Theorem~\ref{thm:ARG-runtime}, already for these problem sizes. Naturally, as $n$ increases, the curves are smoother and the optimal heuristic is used more often. It is shown that each of the optimal heuristics is used most often in the regions in which they are optimal. In the border regions, the best two heuristics are used approximately half of the time each while the other heuristics are rarely used. Hence, the desired heuristic use is achieved even for larger heuristic sets.

\begin{figure}[t!]
\centering
\pgfplotsset{
	xmin=0, xmax=100,ymin=0, ymax=100,
	width=0.37\linewidth,height=1.6in,scale only axis,
	tick label style={/pgf/number format/.cd,fixed,fixed zerofill,precision=0},
	ytick={0,25,50,75,100},
	ymajorgrids=true,
	xtick placement tolerance=-0.1pt,
}
\pgfplotstableread{Experiments2/fig3a.tsv}\dtA
\pgfplotstableread{Experiments2/fig3b.tsv}\dtB
\pgfplotstableread{Experiments2/fig3c.tsv}\dtC
\pgfplotstableread{Experiments2/fig3d.tsv}\dtD
\begin{tabular}{cc}
\begin{tikzpicture}[trim axis right]
\begin{axis}[xlabel=\% $\textsc{LO}(x)$, ylabel=Operator use (\% iterations)]
\addplot[red,dashed]   table[x index=0,y index=1] {\dtA};
\addplot[blue,dashed]  table[x index=0,y index=2] {\dtA};
\addplot[red,dotted]  table[x index=0,y index=3] {\dtA};
\addplot[blue,dotted] table[x index=0,y index=4] {\dtA};
\addplot[red,thick]  table[x index=0,y index=5] {\dtA};
\addplot[blue,thick] table[x index=0,y index=6] {\dtA};
\end{axis}
\end{tikzpicture}
 &
\begin{tikzpicture}[trim axis left,trim axis right]
\begin{axis}[xlabel=\% $\textsc{LO}(x)$, yticklabels={,,}]
\addplot[red,thick]     table[x index=0,y index=7] {\dtB};
\addplot[blue,thick]    table[x index=0,y index=8] {\dtB};
\addplot[dgreen,thick]  table[x index=0,y index=9] {\dtB};
\addplot[red,dashed]    table[x index=0,y index=1] {\dtB};
\addplot[blue,dashed]   table[x index=0,y index=2] {\dtB};
\addplot[dgreen,dashed] table[x index=0,y index=3] {\dtB};
\addplot[red,dotted]    table[x index=0,y index=4] {\dtB};
\addplot[blue,dotted]   table[x index=0,y index=5] {\dtB};
\addplot[dgreen,dotted] table[x index=0,y index=6] {\dtB};

\end{axis}
\end{tikzpicture}\\
(a) $k=2$ & (b) $k=3$ \\[6pt]
\begin{tikzpicture}[trim axis right]
\begin{axis}[xlabel=\% $\textsc{LO}(x)$,ylabel=Operator use (\% iterations)]

\addplot[red,thick]     table[x index=0,y index=9]  {\dtC};
\addplot[blue,thick]    table[x index=0,y index=10] {\dtC};
\addplot[dgreen,thick]  table[x index=0,y index=11] {\dtC};
\addplot[orange,thick]  table[x index=0,y index=12] {\dtC};

\addplot[red,dashed]    table[x index=0,y index=1] {\dtC};
\addplot[blue,dashed]   table[x index=0,y index=2] {\dtC};
\addplot[dgreen,dashed] table[x index=0,y index=3] {\dtC};
\addplot[orange,dashed] table[x index=0,y index=4] {\dtC};
\addplot[red,dotted]    table[x index=0,y index=5] {\dtC};
\addplot[blue,dotted]   table[x index=0,y index=6] {\dtC};
\addplot[dgreen,dotted] table[x index=0,y index=7] {\dtC};
\addplot[orange,dotted] table[x index=0,y index=8] {\dtC};

\end{axis}
\end{tikzpicture}
& 
\begin{tikzpicture}[trim axis left,trim axis right]
\begin{axis}[xlabel=\% $\textsc{LO}(x)$, yticklabels={,,},
	legend columns=5,legend to name={fig3legend}]
\addplot[red,thick]     table[x index=0,y index=11]  {\dtD};
\addplot[blue,thick]    table[x index=0,y index=12]  {\dtD};
\addplot[dgreen,thick]  table[x index=0,y index=13]  {\dtD};
\addplot[orange,thick]  table[x index=0,y index=14]  {\dtD};
\addplot[purple,thick]  table[x index=0,y index=15]  {\dtD};

\addplot[red,dashed]     table[x index=0,y index=1]  {\dtD};
\addplot[blue,dashed]    table[x index=0,y index=2]  {\dtD};
\addplot[dgreen,dashed]  table[x index=0,y index=3]  {\dtD};
\addplot[orange,dashed]  table[x index=0,y index=4]  {\dtD};
\addplot[purple,dashed]  table[x index=0,y index=5]  {\dtD};

\addplot[red,dotted]    table[x index=0,y index=6]  {\dtD};
\addplot[blue,dotted]   table[x index=0,y index=7]  {\dtD};
\addplot[dgreen,dotted] table[x index=0,y index=8]  {\dtD};
\addplot[orange,dotted] table[x index=0,y index=9]  {\dtD};
\addplot[purple,dotted] table[x index=0,y index=10] {\dtD};
\legend{RLS$_1$, RLS$_2$, RLS$_3$, RLS$_4$, RLS$_5$}
\end{axis}
\end{tikzpicture}\\
(c) $k=4$ & (d) $k=5$ \\[6pt]
\end{tabular}\hfill \vspace{2mm} \ref{fig3legend}

\caption{Percentage of the iterations ARG using $H=\{\rlsone,\dots,\rlsk\}$ and $\sigma=\sqrt{n}/\ln n$ applies each mutation operator for, average over 100 runs for $n=10^5$ (dotted), $n=10^7$ (dashed) and $n=10^9$ (solid).}
\label{fig:average-opt-use}
\end{figure}

\section{Conclusion}
Recently it has been proven that a Random Gradient selection hyper-heuristic runs in optimal expected time for \LeadingOnes if the learning period $\tau$ is set appropriately. In this paper, we have presented an Adaptive Random Gradient (ARG) hyper-heuristic that automatically adjusts the learning period throughout the run. ARG uses an innovative $1-o(1)$ self-adjusting rule that strives to adapt the learning period such that a failure occurs approximately every $\omega(1)$ successes. The novelty consists of seeking many successes before a failure in contrast to a success after many failures of traditional adaptive algorithms~\citep{KernEtAl2004,DoerrDoerr2015}.

We have rigorously proved that ARG optimises \LeadingOnes in the best expected runtime achievable using the low-level heuristic set  $H=\{\rlsone,\dots,\rlsk\}$ for any constant $k=\Theta(1)$, up to lower order terms. Our proof also shows that with probability $1-o(1)$, only a fraction of $o(1)$ of the iterations use any non-optimal heuristic. Hence, the optimal heuristic is used most of the time as desired. We believe that also a non-constant number of heuristics (e.g., a logarithmic number) could be handled by our algorithm, again giving optimal performance apart from lower-order terms. Furthermore, ARG outperforms any unbiased (1+1) black box algorithm using any strict subset of $H$.

We have complemented the theory with experiments for practical problem sizes. The results show that the parameter $\sigma$, indicating the ratio of successes to failures which maintains a stable $\tau$ value, is fairly robust. If it is set within the range of values predicted by our theoretical analysis, then ARG outperforms the best hyper-heuristics with fixed $\tau$ reported in the literature.

Recently, \cite{LissovoiEtAl2020AAAI} have shown that the adaption of the learning period in ARG is provably useful for two function classes with different characteristics. They proved that ARG has optimal asymptotic runtime for the \textsc{Ridge} function class (where the HH must learn to use \rlsone throughout the run) and the \textsc{OneMax} function class (where different low-level heuristics are preferable in different areas of the search space).

Future work should evaluate the performance of ARG for a wider range of problems including ones from combinatorial optimisation and real-world applications. Furthermore, implementing the adaptive framework used for ARG into other hyper-heuristic approaches should be considered. For example, the Move Acceptance Hyper-heuristic, which has shown effective performance for some multi-modal problems with local and global mutation operators \citep{LissovoiEtAl2023,DoerrEtAl2023,DoerrLutzeyer2024,DoerrIJCAI2025} contains a fixed parameter which controls the tradeoff between elitist and non-elitism moves; automatically adapting this parameter could reduce the performance gap between problems with different landscapes. 

\paragraph*{Acknowledgements}
{\footnotesize This work was initiated in May 2017 while Pietro S.~Oliveto was a \emph{chercheur invit\'e} at \'Ecole Polytechnique. This work was supported by a public grant as part of the Investissement d'avenir project, reference ANR-11-LABX-0056-LMH, LabEx LMH, in a joint call with Gaspard Monge Program for optimization, operations research and their interactions with data sciences and by EPSRC under grant EP/M004252/1.

We are grateful to Andrei Lissovoi for his help in developing the original conference paper.}

\clearpage

\bibliographystyle{apalike}
\bibliography{references}

}
\end{document}